\documentclass{article}

\PassOptionsToPackage{square,sort,comma,numbers}{natbib}

\usepackage[preprint]{neurips_2024}

\usepackage{amsmath,amsfonts,bm}

\def\eqref#1{equation~\ref{#1}}

\def\1{\bm{1}}

\def\rvn{{\mathbf{n}}}

\def\rvw{{\mathbf{w}}}
\def\rvx{{\mathbf{x}}}
\def\rvy{{\mathbf{y}}}
\def\rvz{{\mathbf{z}}}

\def\rmA{{\mathbf{A}}}

\def\rmI{{\mathbf{I}}}

\DeclareMathAlphabet{\mathsfit}{\encodingdefault}{\sfdefault}{m}{sl}
\SetMathAlphabet{\mathsfit}{bold}{\encodingdefault}{\sfdefault}{bx}{n}

\newcommand{\R}{\mathbb{R}}

\newcommand{\KL}{D_{\mathrm{KL}}}
\newcommand{\TV}{D_{\mathrm{TV}}}

\newcommand{\bx}{\mathbf{x}}
\newcommand{\bX}{\mathbf{X}}
\newcommand{\bE}{\mathbb E}
\newcommand{\poly}{\mathrm{poly}}

\usepackage[utf8]{inputenc} 
\usepackage[T1]{fontenc}    
\usepackage{hyperref}       
\usepackage{url}            
\usepackage{booktabs}       
\usepackage{amsfonts}       
\usepackage{nicefrac}       
\usepackage{microtype}      
\usepackage{xcolor}         
\usepackage{graphicx}
\usepackage{tabularx}
\usepackage{hyperref}
\usepackage{amsmath}
\usepackage{amssymb}
\usepackage{mathtools}
\usepackage{amsthm}
\usepackage{caption}
\usepackage{subcaption}
\usepackage{adjustbox}
\usepackage{multirow}
\usepackage{siunitx}
\usepackage{wrapfig}
\usepackage{lipsum}
\usepackage{times}
\usepackage{epsfig}
\usepackage{makecell}
\usepackage{multirow}
\usepackage{algorithm}
\usepackage{algorithmic}
\usepackage[capitalize,noabbrev]{cleveref}
\usepackage[textsize=tiny]{todonotes}
\usepackage[toc,page,header]{appendix}
\usepackage{minitoc}
\usepackage{titletoc}
\usepackage{color, colortbl}
\definecolor{gray}{gray}{0.9}

\theoremstyle{plain}
\newtheorem{theorem}{Theorem}[section]

\newtheorem{lemma}[theorem]{Lemma}

\theoremstyle{definition}

\newtheorem{assumption}[theorem]{Assumption}
\theoremstyle{remark}

\title{Think Twice Before You Act: Improving Inverse Problem Solving With MCMC}

\author{Yaxuan Zhu \\
    \small{UCLA}\\
    \small{yaxuanzhu@g.ucla.edu} \\
    \And Zehao Dou \\
    \small {Yale University} \\
    \small {zehao.dou@yale.edu} \\ 
    \And Haoxin Zheng \\
    \small{UCLA}\\
    \small{haoxinzheng@g.ucla.edu } \\
    \And Yasi Zhang \\
    \small {UCLA} \\
    \small{yasminzhang@g.ucla.edu} \\
    \And Ying Nian Wu \\
    \small{UCLA}\\
    \small{ywu@stat.ucla.edu}\\ 
    \And Ruiqi Gao\\
    \small{Google DeepMind}\\
    \small{ruiqig@google.com}\\
}

\begin{document}

\maketitle

\begin{abstract}

Recent studies demonstrate that diffusion models can serve as a strong prior for solving inverse problems. A prominent example is Diffusion Posterior Sampling (DPS), which approximates the posterior distribution of data given the measure using Tweedie's formula. Despite the merits of being versatile in solving various inverse problems without re-training, the performance of DPS is hindered by the fact that this posterior approximation can be inaccurate especially for high noise levels.
Therefore, we propose \textbf{D}iffusion \textbf{P}osterior \textbf{MC}MC (\textbf{DPMC}), a novel inference algorithm based on Annealed MCMC to solve inverse problems with pretrained diffusion models. We define a series of intermediate distributions inspired by the approximated conditional distributions used by DPS. 
Through annealed MCMC sampling, we encourage the samples to follow each intermediate distribution more closely before moving to the next distribution at a lower noise level, and therefore reduce the accumulated error along the path.
We test our algorithm in various inverse problems, including super resolution, Gaussian deblurring, motion deblurring, inpainting, and phase retrieval. Our algorithm outperforms DPS with less number of evaluations across nearly all tasks, and is competitive among existing approaches. 
         
\end{abstract}

\section{Introduction}
Diffusion Models \cite{Sohl-DicksteinW15, DDPM_HoJA20, SongE19, Score_base_0011SKKEP21} have recently achieved significant success in high-dimensional data generation, including images \cite{DhariwalN21, ramesh2022hierarchical, rombach2022high, SahariaCSLWDGLA22}, videos \cite{HoSGC0F22, sorareport, ho2022imagen, blattmann2023align, girdhar2023emu}, audio \cite{kong2020diffwave, chen2020wavegrad}, text \cite{LiTGLH22}, and and 3D generation \cite{zero123_LiuWHTZV23, PooleJBM23, wu2023reconfusion, shi2023mvdream, gao2024cat3d}. 
Beyond generation, recent works have applied diffusion models to solve inverse problems in a plug-and-play fashion without the need for fine-tuning \cite{JalalADPDT21, KadkhodaieS21, kawar2022denoising, Score_base_0011SKKEP21, ChoiKJGY21, ChungSRY22, DPS_ChungKMKY23, dou2023diffusion, mardani2023variational, SongVMK23, RoutRDCDS23, song2023solving, ChungY22, FengSRCBF23, ZhuZLCWTG23, Bridge_ChungKY23}, where the goal is to restore data $\rvx$ from degraded measurements $\rvy$. Among these, one line of works \cite{JalalADPDT21, DPS_ChungKMKY23, SongVMK23, RoutRDCDS23, song2023solving} proposes to modify the inference process of diffusion models with guidance that encourages samples to be consistent with the measure. 
Other  perspectives include variational inference \cite{mardani2023variational, FengSRCBF23, ZhuZLCWTG23}, Bayesian filtering \cite{dou2023diffusion}, and solving an inner loop optimization problem during sampling~\cite{Song20220011S0E22, song2023solving}. 

A typical challenge in this context is that the posterior distribution  $p(\rvx|\rvy) $ is only defined for clean samples $ \rvx $, yet during sampling, an estimation of $ p(\rvx_t|\rvy) $ is needed at each diffusion step $ t $. DPS \cite{DPS_ChungKMKY23} approximates the intractable posterior through Tweedie's formula \cite{efron2011tweedie}, enabling its application in general inverse problems. However, this approximation might be inaccurate especially on high noise levels, leading to samples of low quality. To address this, $\Pi$GDM \cite{SongVMK23} attempts to improve the approximated guidance by pseudo-inverting the measurement. Additionally, \cite{RoutRDCDS23, song2023solving} aim to enhance DPS in the Latent Diffusion model \cite{rombach2022high} by introducing extra guidance terms or resampling processes to the data space.

In this work, we propose to leverage annealed Markov Chain Monte Carlo (MCMC) to improve solving inverse problems with diffusion models.
As a family of effective algorithms that draw samples from complex distributions, MCMC is widely used in training Energy-Based Models (EBMs) \cite{XieLZW16, NijkampHZW19, DuLTM21} and sampling from diffusion models \cite{SongE19, Score_base_0011SKKEP21}. Annealed MCMC further proposes to facilitate MCMC sampling by gradually sampling from a sequence of distributions from decreasingly lower temperature, to accelerate the mixing.  
Previous work \cite{DuDSTDFSDG23} has leveraged Annealed MCMC in compositional generation with EBMs or diffusion models. In this work, we propose using MCMC to reduce the error of posterior approximation in solving inverse problems.


There are some earlier works \cite{JalalADPDT21, Score_base_0011SKKEP21} that employ annealed MCMC with Langevin dynamics to solve inverse problems. They tackle the intractability of the posterior distribution by projecting the current sample onto the measurement subspace. However, these approaches might fail when measurements are noisy or the measurement process is non-linear, as discussed in~\cite{DPS_ChungKMKY23}. Moreover, a single MCMC step is executed at every noise level before moving to the next noise level.
Alternatively, we propose to build the the intermediate distributions of our annealed MCMC with the approximated posterior distributions derived in~\cite{DPS_ChungKMKY23}. Despite the fact that such approximation might be inaccurate for the data distributions defined by the forward diffusion, they still remain as a valid sequence of distributions that can be leveraged in annealed MCMC. We further propose to run multiple sampling steps at each noise level similar to \cite{DuDSTDFSDG23}, that empirically improves the performance.


In summary, we make the following contributions in this work:
\begin{itemize} 
    \item We propose the \textbf{D}iffusion \textbf{P}osterior \textbf{MC}MC (\textbf{DPMC}) algorithm, which leverages annealed MCMC with a sequence of posterior distribution of data given measurements, whose formula is inspired by DPS~\cite{DPS_ChungKMKY23}. 
    \item We demonstrate that empirically DPMC outperforms DPS in terms of sample quality across various types of inverse problems in image domain, including both linear and nonlinear inverse problems. DPMC also establishes competitive performance compared with other existing approaches.
    \item Through extensive ablation study and comparison with other approaches, we highlight the effectiveness of MCMC-based approaches for solving inverse problems, and a constant improvement of performance with increasingly larger number of sampling step. 
    
\end{itemize}

\section{Background}
\subsection{Inverse Problem}
\label{sec:inv_problem}
We denote data distribution $\rvx \sim p_{data}(\rvx)$. In many scientific applications, instead of directly observing $\rvx$, we might only have a partial measurement $\rvy$, which is derived from $\rvx$, and we want to restore $\rvx$ from $\rvy$. Formally, we might assume the following mapping between $\rvx$ and $\rvy$
\begin{equation}
    \rvy = \rmA(\rvx) + \rvn, \; \rvx \in \R^D, \; \rvy \in \R^d, \; \rvn \sim \mathcal{N}(0, \sigma^2 \rmI)
\end{equation}
where $\rmA(\cdot): \R^D \mapsto \R^d$ is called forward measurement operator and $\rvn$ is the measurement noise following Gaussian distribution. Therefore, we have $p(\rvy|\rvx) \sim  \mathcal{N}(\rmA(\rvx), \sigma^2 \rmI)$. Mapping between $\rvx$ to $\rvy$ can be many-to-one. This makes exactly restoring $\rvx$ become an ill-posed problem.

\subsection{Diffusion Models}
Diffusion models \citep{Sohl-DicksteinW15, DDPM_HoJA20, Song20220011S0E22, SongE19} define a generative process that gradually transforms a random noise distribution into a clean data distribution. A diffusion model consists of a forward noise injection process and a backward denoising process. Let $\rvx_0 \sim p_{data}$ denote clean observed samples. DDPM \cite{DDPM_HoJA20} defines a Markovian forward process as follows:
\begin{align}
     q(\rvx_{1:T} | \rvx_{t-1}) &= q(\rvx_{t} | \rvx_{t-1}) = \mathcal{N}( \sqrt{\alpha_{t}} \rvx_{t-1}, \beta_{t} \rmI) \nonumber \\
     q(\rvx_t| \rvx_0) &= \mathcal{N}(\sqrt{\bar{\alpha}_t} \rvx_0, (1 - \bar{\alpha}_t)\rmI) 
\end{align}
where $\{\beta_t\}_{t=1}^{T}$ is the manually designed noise schedule that might differ from different works \citep{DDPM_HoJA20,Score_base_0011SKKEP21, KarrasAAL22}. And $\alpha_t = 1 - \beta_t$, $\bar{\alpha}_t = \prod_{i=1}^{t}\alpha_i$. In the backward process, we start from the noise distribution and gradually denoise the samples as follows:
\[\rvx_{t-1} = \mu_\theta(\rvx_t, t) + \sqrt{\tilde{\beta}_t} \rvz_t, \;\; \rvz_t \sim \mathcal{N}(0, \rmI)\]
\begin{equation}
    \mu_\theta(\rvx_t, t) = \frac{1}{\sqrt{1 - \beta_t}} \left(\rvx_t - \frac{\beta_t}{\sqrt{1 - \bar{\alpha}_t}} \epsilon_\theta(\rvx_t, t)\right), \;\; \tilde{\beta}_t = \frac{1 - \bar{\alpha}_{t-1}}{1 - \bar{\alpha}_{t}} \beta_t 
\end{equation}
where $\epsilon_\theta(\rvx_t, t)$ is parameterized by a neural network \citep{DDPM_HoJA20, PeeblesX23}. Let the marginal distribution defined by the forward process of $\rvx_t$ be denoted as $p_t(\rvx_t)$. When trained with denoising score matching loss \citep{DDPM_HoJA20, Score_base_0011SKKEP21}, with sufficient data and model capacity, for almost all $\rvx$ and $t$, $\epsilon_\theta(\rvx_t, t)$ corresponds to the gradient field $\nabla_{\rvx_t} \log p_t(\rvx_t)$ as follows:
\begin{align}
    \nabla_{\rvx_t} \log p_t(\rvx_t) = -\frac{\epsilon_\theta(\rvx_t, t)}{\sqrt{1 - \bar{\alpha}_t}}
    \label{eq:score_eps}
\end{align}
Other than DDPM, DDIM\cite{DDIM_SongME21} defines a Non-Markovian forward process that shares the same training objective as DDPM. Thus, a model trained with DDPM can be directly applied in the DDIM sampler to accelerate the sampling process. A DDIM sampler follows:
\begin{align}
    \rvx_{t-1} = \sqrt{\bar{\alpha}_{t-1}} \left(\frac{\rvx_t - \sqrt{1 - \bar{\alpha}_{t}}\epsilon_\theta(\rvx_t, t)}{\sqrt{\bar{\alpha}_{t}}}\right) + \sqrt{1 - \bar{\alpha}_{t-1} - \sigma_t^2} \epsilon_\theta(\rvx_t, t) + \sigma_t \rvz_t
    \label{eq:diffusion_backward_ddim}
\end{align}
where $\rvz_t \sim \mathcal{N}(0, \rmI)$ and variance $\sigma_t$ can be arbitrary defined.

\subsection{Diffusion Posterior Sampling}
To solve the ill-posed inverse problem, Diffusion Posterior Sampling (DPS) \citep{DPS_ChungKMKY23} recruits pretrained diffusion models \citep{DDPM_HoJA20, Score_base_0011SKKEP21} as prior and propose an iterative optimization algorithm. According to Bayes rule, the gradient field of posterior distribution $\nabla_{\rvx_t} \log p(\rvx_t|y) = \nabla_{\rvx_t} \log p_t(\rvx_t) + \nabla_{\rvx_t} \log p(\rvy | \rvx_t)$ where the term $\nabla_{\rvx_t} \log p_t(\rvx_t)$ is estimated by the pretrained diffusion model using equation \ref{eq:score_eps}. In DPS, the authors adapt the approximation $p(\rvy|\rvx_t) \approx p(\rvy|\hat{\rvx}_0)$, where $\hat{\rvx}_0(\rvx_t) = E_{\rvx_0 \sim p(\rvx_0 | \rvx_t )}[\rvx_0]$ can be estimated through Tweedie's formula \cite{efron2011tweedie}:
\begin{align}
    \hat{\rvx}_0 = \frac{\rvx_t - \sqrt{1 - \bar{\alpha}_{t}}\epsilon_\theta(\rvx_t, t)}{\sqrt{\bar{\alpha}_{t}}}
    \label{eq:tweedie}
\end{align}

Given noisy observation data $x_t$, DPS makes the following updates
\begin{align}
    \rvx_{t-1}' &= \frac{1}{\sqrt{1 - \beta_t}} (\rvx_t - \frac{\beta_t}{\sqrt{1 - \bar{\alpha}_t}} \epsilon_\theta(\rvx_t, t)) + \sigma_t \rvz_t, \; \rvz_t \sim \mathcal{N}(0, \rmI) \nonumber \\
    &\rvx_{t-1} = \rvx_{t-1}' - \zeta_t\nabla_{\rvx_t} \| \rvy - \rmA(\hat{\rvx}_0) \|^2_2 
    \label{eq:dps_update}
\end{align}
In practice, DPS employs the following adaptive step size parameter $\zeta_t$:  
\begin{equation}
    \label{eq:dps_zeta}
    \zeta_t = \frac{\zeta}{\| \rvy - \rmA(\hat{\rvx}_0) \|_2}
\end{equation} 

where $\zeta$ is a fix constant.

\section{Diffusion Posterior with MCMC sampling} 
In this section, we introduce our DPMC model, which combines diffusion models and Markov Chain Monte Carlo (MCMC) sampling. The former one is for the progressive denoising to provide an unconditional proposal distribution. The latter is for the conditional guidance by the measurement $\rvy$. According to \textit{Theorem 1} in \cite{DPS_ChungKMKY23}, error of the approximation $ p(\rvy|\rvx_t) \approx p(\rvy|\hat{\rvx}_0)$ is bounded by the estimation error between $\rvx_0$ and $\hat{\rvx}_0$. When the noise level is low, $ p(\rvx_0 | \rvx_t) $ can be single-modal and the estimation $\hat{\rvx}_0$ and the posterior approximation might be accurate. However, at higher noise levels, where $ p(\rvx_0 | \rvx_t) $ can be indeed multi-modal and $\hat{\rvx}_0$ can be far from $\rvx_0$, this approximation can be too loose, potentially leading to inferior results. As illustrated in Figure \ref{fig:illustration_DPS}, while the samples of DPS are valid, they might fail to capture vivid details. To address these problems, we instead resort to MCMC sampling.
\label{sec:DPM}
\begin{figure}
    \centering 
    \includegraphics[width=.98\textwidth]{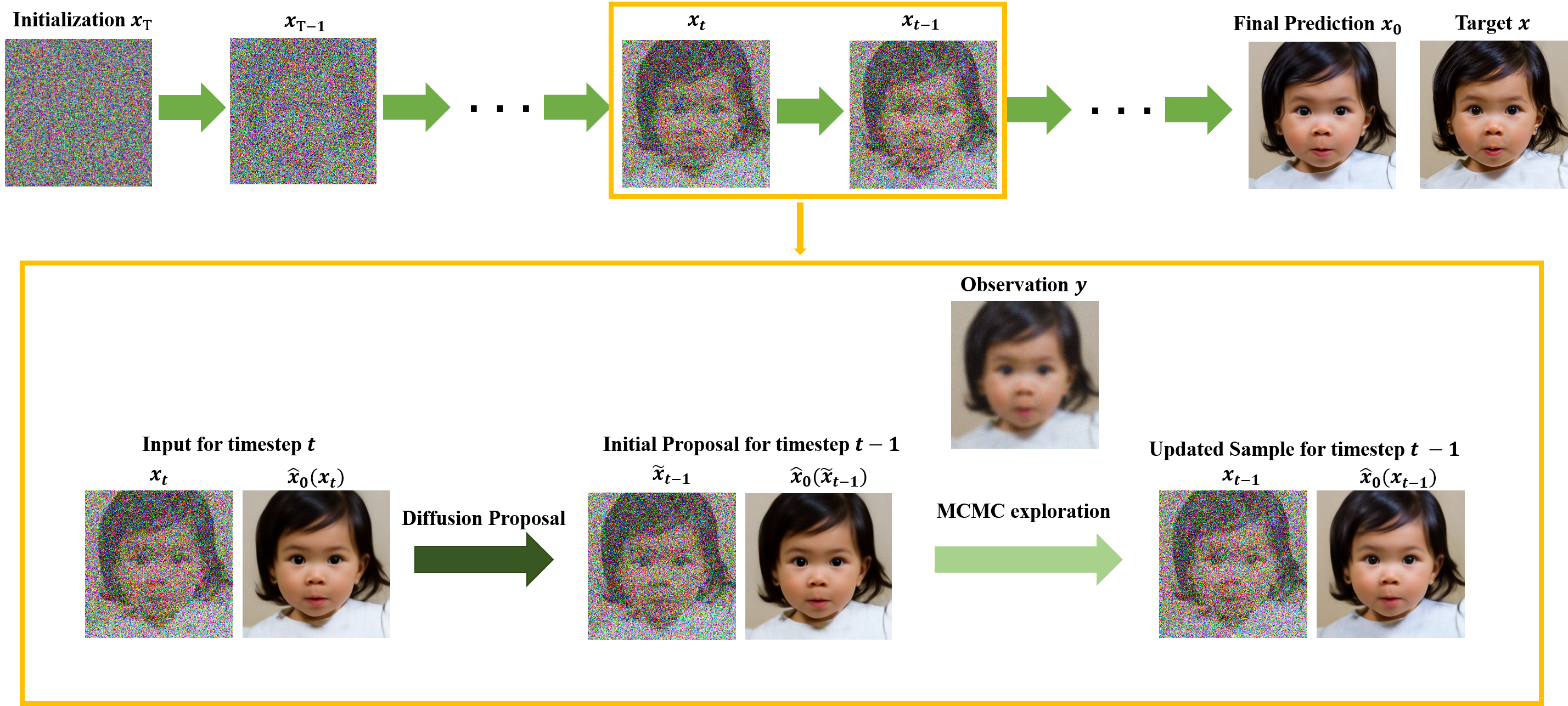}
    \caption{The illustration of DPMC. At each step, DPMC iterates over diffusion proposal step and MCMC exploration step.}
    \label{fig:diagram}
\end{figure}

Unlike diffusion models, which explicitly define all the noise distributions $ p_t(\rvx_t) $ through the forward process, MCMC sampling does not require knowledge of the sample distributions at each intermediate MCMC step. As long as a valid transition kernel and a sufficient number of MCMC sampling steps are used, MCMC is known to sample from arbitrary underlying distributions. 

\begin{wrapfigure}{r}{0.55\textwidth} 
\begin{minipage}{0.55\textwidth}
        \vspace{-2.0em}
        \begin{algorithm}[H]
        \caption{DPMC Algorithm}
        \label{alg:dpm}
        \begin{algorithmic}[1]
        \REQUIRE Inverse problem forward operator $\mathcal{A}(.)$, noisy measurement $\mathbf{y}$, pretrained diffusion prior model $p_t(\mathbf{x}_t)$, number of intermediate noise levels $T$, number of MCMC steps $K$, MCMC step size $\eta_t$, weighting parameter $\xi_t$.
        \STATE $\mathbf{x}_T \sim \mathcal{N}(0, \mathbf{I})$
        \FOR{$t = T$ to $1$}
            \STATE \textbf{Proposal step:} Sample $\tilde{\mathbf{x}}_{t-1}$ from $\mathbf{x}_t$ following Eq.~\ref{eq:diffusion_backward_ddim}
            \STATE \textbf{Exploration step:} Set $\mathbf{x}_{t-1}^{(0)} = \tilde{\mathbf{x}}_{t-1}$
            \FOR{$k = 1$ to $K$}
                \STATE Update $\mathbf{x}^{(k)}_{t-1}$ to $\mathbf{x}^{(k+1)}_{t-1}$ following Eq.~\ref{eq:Langevin}
            \ENDFOR
        \ENDFOR
    \end{algorithmic}
    \end{algorithm}
    \end{minipage}
\end{wrapfigure}


Annealed MCMC \cite{neal2001annealed, SongE19, Score_base_0011SKKEP21} is a widely used technique to accelerate the MCMC sampling process for highly multi-modal data. In the annealed sampling process, samples gradually progress through a series of intermediate distributions with different temperatures. MCMC is applied at each intermediate distribution to enable samples to transition from the previous distribution to the current one. In this work, we apply annealed MCMC to solve the inverse problem. We introduce a series of intermediate distributions $\{\tilde{p}_t(\rvx_t | \rvy)\}_{t=1}^{T}$. Note that we do not expect $\tilde{p}_t(\rvx_t | \rvy)$ to be close to the true posterior distribution $p_t(\rvx_t|\rvy)$ at every intermediate distribution. Instead, we only require $\tilde{p}_t(\rvx_t | \rvy)$ to agree with $p_t(\rvx_t|\rvy)$ at the clean image distribution. The key idea is that MCMC sampling can bridge the gap between different intermediate distributions. And the intermediate distributions only need to form a trajectory that enables MCMC to smoothly transition from the noise distribution to the target distribution in the clean data space.
  
Following \cite{DPS_ChungKMKY23}, we define the intermediate distributions $\tilde{p}_t(\rvx_t | y)$ as
\begin{align}
    \tilde{p}_t(\rvx_t | \rvy) \propto p_t(\rvx_t) \exp(-\rho \| \rvy - \rmA(\hat{\rvx}_0) \|_2^2)
\end{align}
where $\rho=1 / \sigma^2$. $p_t(\rvx_t)$ denotes the diffusion prior at noise level $t$ and $\nabla_{\rvx_t} \log p_t(\rvx_t)$ can be estimated using equation \ref{eq:score_eps}.  However, instead of using \ref{eq:dps_update} for sampling, which requires the intermediate distribution to be sufficiently close to the ground truth posterior at each noise level, we propose a new proposal-and-update algorithm based on annealed MCMC sampling. 

\textbf{Proposal Stage:} Given samples $\rvx_t$ that follow the intermediate distribution $\tilde{p}_t(\rvx_t | \rvy)$, we first denoise them to $t - 1$ following the standard diffusion step without considering the extra guidance. We denote the proposed sample as $\tilde{\rvx}_{t-1}$. The proposal step aims to provide a good initialization for the intermediate distribution at $\tilde{p}_{t-1}(\rvx_{t-1} | \rvy)$ with the help of the diffusion model. The samples $\tilde{\rvx}_{t-1}$ might not fully adhere to the target distribution $\tilde{p}_{t-1}(\rvx_{t-1} | \rvy)$, but we hope they are close enough to the target distribution so that we can obtain good samples with a few MCMC sampling steps.

\textbf{Exploration Stage:}  We then encourage the samples to explore the landscape at noise level $t-1$ and follow $\tilde{p}_{t-1}(\rvx_{t-1}| \rvy)$ with MCMC updates. We employ  Langevin Dynamics \citep{roberts1996exponential} as the transition kernel. Specifically, starting from $\rvx_{t-1}^0 = \tilde{\rvx}_{t-1}$, we iterate the following updates:
\begin{align}
    \rvx_{t-1}^{k+1} = \rvx_{t-1}^{k} + \eta_{t-1}\nabla_{\rvx_{t-1}^{k}} \log \tilde{p}_{t-1}(\rvx_{t-1}^{k} | \rvy) + \sqrt{2\eta_{t-1}} \rvw, \; \rvw \sim \mathcal{N}(0, \rmI) 
    \label{eq:Langevin}
\end{align}
We named our algorithm \textbf{D}iffusion \textbf{P}osterior \textbf{MC}MC (\textbf{DPMC}). We illustrate DPMC Figure \ref{fig:diagram} and in Algorithm \ref{alg:dpm}. We also provide a theoretical analysis which shows that the conditional distribution derived by DPMC is $\varepsilon$-close to the ground truth, i.e.
\[\TV(q_0(\rvx_0\mid \rvy)~\|~p^*(\rvx_0\mid \rvy)) \leq \varepsilon\]
with a carefully chosen step size $\eta_t$ and inner loop $K$ of the Langevin MCMC algorithm, under several assumptions on the score estimation error, conditional probability approximation error as well as some convexity and Lipschitz continuity conditions. Detailed assumptions and main theorem for the convergence of our DPMC algorithm is deferred to Appendix \ref{sec:theory}.


\begin{figure}
    \centering 
    \includegraphics[width=.85\textwidth]{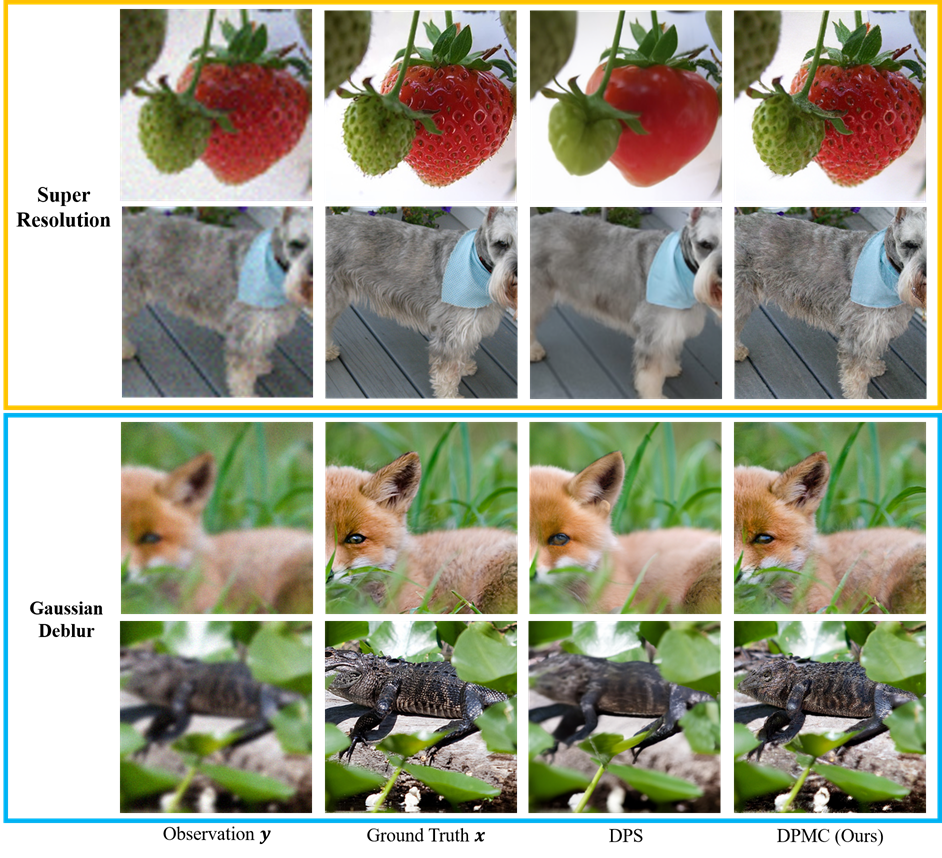}
    \caption{Qualitative comparison between DPS samples and our DPMC samples. Images in the top rows are from super resolution task. Images in the bottom rows are from Gaussian deblurring task.}
    \label{fig:illustration_DPS}
\end{figure}

\noindent {\bf Implementation Details} ~
We apply the same implementation as \cite{DPS_ChungKMKY23} in equation \ref{eq:dps_zeta} by introducing a step size parameter $\zeta_t$ to be inverse proportional to optimization distance $\| \rvy - \rmA(\hat{\rvx}_0) \|_2$. This setting equals to sampling according to the prior distribution tilted by an exponential distribution of the optimization distance. In practice, we calculate $\nabla_{\rvx_t} \tilde{p}_t(\rvx_t | \rvy)$ at each step using the following equation:

\begin{align}
    \nabla_{\rvx_t} \log\tilde{p}_t(\rvx_t | \rvy) =  \nabla_{\rvx_t} \log p_t(\rvx_t) -\xi_t \nabla_{\rvx_t} \| \rvy - \rmA(\hat{\rvx}_0) \|_2 
    \label{eq:inter_distribution}
\end{align}
where $\xi_t$ represents a constant or variable weighting schedule. Note that we replace square $l_2$ norm by $l_2$ norm itself. According to \cite{DPS_ChungKMKY23, song2023improved}, using $l_2$ norm might make the optimization to be more robust to outliers and can achieve more stable results than square $l_2$ norm as it imposes a smaller penalty for large errors. 

While adding MCMC exploration might inevitably introduce more sampling steps at each intermediate distribution and thus increase the sampling time, we find that with DPMC, we can counteract this by reducing the number of intermediate distributions needed. In fact, we find that using 200 intermediate distributions is sufficient for DPMC. Moreover, during the early stage of diffusion, the samples are very close to Gaussian noise. Considering that the reason for inserting intermediate distributions is to provide a good initialization for sampling more complex distributions that are very different than Gaussian, we can skip those early stages with high noise levels and start MCMC sampling directly from a moderate noise level. On the other hand, at very low noise levels, $p(\rvx_0|\rvx_t)$ almost becomes single-modal, and the approximation $p(\rvy|\rvx_t) \sim p(\rvy|\rvx_0)$ becomes sufficiently accurate. Thus, we can directly follow DPS without the need for additional MCMC sampling. In fact, we find that DPMC performs well when we apply the proposal and exploration steps only to the middle steps of the sampling process, while following the original DPS setting at both ends. This further reduces the number of evaluation steps needed. In practice, DPMC can achieve much better results compared to DPS with an even smaller number of evaluations (NFE). 

\section{Experiments}
In this section, we conduct experiments on the DPMC method proposed by us, which outperforms existing baselines on diffusion posterior sampling. To start with, we thoroughly introduce the models, datasets and settings of our experiments. 
\subsection{Experiment Settings}
\label{sec:exp_setting}
\textbf{Datasets and Pretrained Model:} Following \cite{DPS_ChungKMKY23}, we test our algorithm on FFHQ $256 \times 256$ dataset \cite{KarrasLA19} and ImageNet $256 \times 256$ dataset \cite{DengDSLL009}. Same as \cite{DPS_ChungKMKY23, dou2023diffusion}, we use 1k validation images for each dataset. For FFHQ dataset, we use the pretrained model provided by \cite{DPS_ChungKMKY23}. For ImageNet dataset, we use the pretrained model provided by \cite{DhariwalN21}. 

\begin{figure}[t!]
    \centering 
    \includegraphics[width=.95\textwidth]{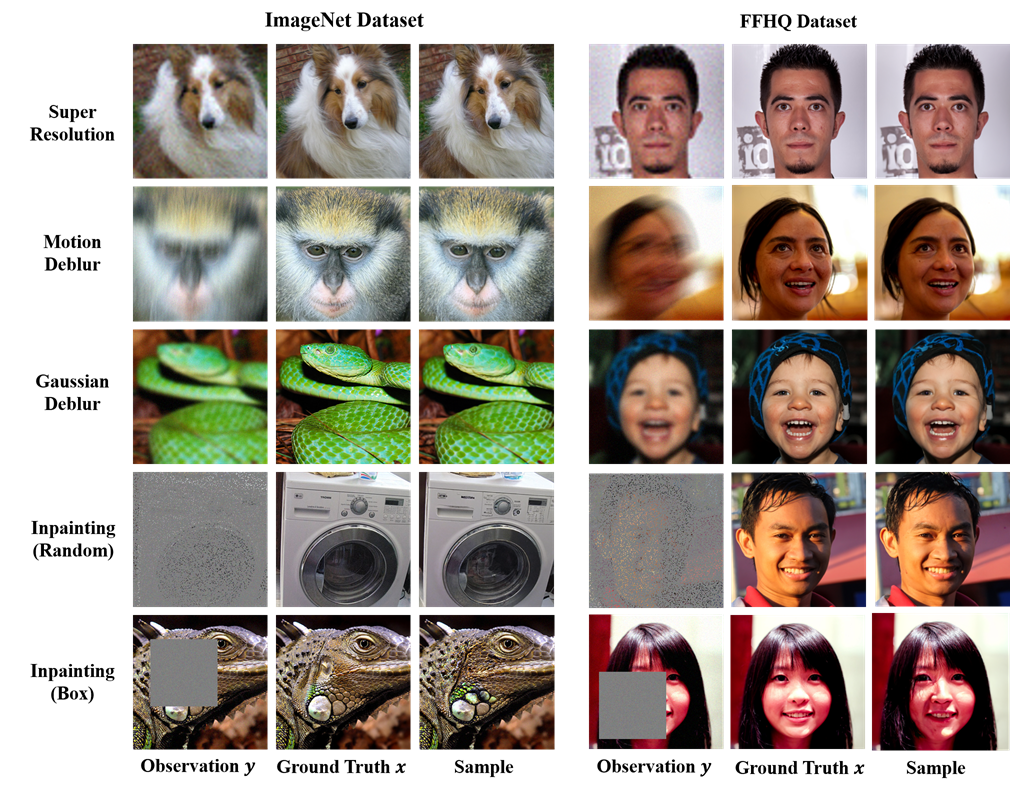}
    \caption{Qualitative results of DPMC on different linear inverse problem tasks.}
    \label{fig:all_samples}
\end{figure}

\textbf{Inverse Problems:} We evaluate the effectiveness of our DPMC algorithm on the following inverse problems: (i) Super-resolution with 4x bicubic downsampling as the forward measurement; (ii) Random inpainting using both box masks and random masks; (iii) Deblurring with Gaussian blur kernels and motion blur kernels \footnote{\href{https://github.com/LeviBorodenko/motionblur}{https://github.com/LeviBorodenko/motionblur}}; (iv) Phase retrieval, where we perform a Fourier transformation on each image and use only the Fourier magnitude as the measurement. Among these inverse problems, (i), (ii), and (iii) are linear inverse problems, while (iv) is a nonlinear inverse problem. We use Gaussian noise with $\sigma = 0.05$ for all tasks and set task-related parameters according to \cite{DPS_ChungKMKY23, dou2023diffusion} to ensure a fair comparison.

\textbf{Hyper-Parameter Setting:} 
We use the DDIM sampler with default variance $\sigma_t = 0$ as our diffusion proposal sampler. We use $T = 200$ intermediate distributions and $K = 4$ MCMC steps at each intermediate distribution. Empirically, we set $\eta_t = \eta \beta_t$ and $\xi_t = \xi \bar{\alpha}^3_t$, where $\eta$ and $\xi$ are task-related constants. We find these schedules work well across all settings. The detailed parameter settings can be found in Section \ref{sec:exp_detail} in the Appendix. As discussed in \ref{sec:DPM}, we apply the proposal-exploration step in the middle 60\% of the total sampling steps and resort to original DPS sampler step at the first 30\% and the last 10\% noise levels. An exception is the inpainting task with a box-type mask on the ImageNet dataset, where we find that applying the proposal-exploration step until clean images yields better results. Considering both the initial proposal step and the MCMC sampling steps, the NFE of this setting is around 700, which is smaller than the 1000 NFE of DPS. All of the experiments are carried out on a single Nvidia A100 GPU. We report the running clock time of DPMC in Table \ref{tab:running time} in Appendix \ref{sec:more_exp}.

\textbf{Baselines:}
On linear inverse problems, we compare DPMC algorithm with the original DPS\cite{DPS_ChungKMKY23}, filter posterior sampling (FPS) with and without sequential Monte Carlo sampling\cite{dou2023diffusion}, denoising diffusion
restoration models (DDRM) \cite{kawar2022denoising}, manifold constrained gradients (MCG) \cite{ChungSRY22}, plug-and-play alternating direction method of multiplier (PnP-ADMM) \cite{ChanWE17}, Score-Based SDE \cite{Score_base_0011SKKEP21, ChoiKJGY21} and alternating direction method of multiplier with total-variation (TV) sparsity regularized optimization (ADMM-TV). On phase retrieval, we compare DPMC with DPS\cite{DPS_ChungKMKY23}, oversampling
smoothness (OSS)\cite{rodriguez2013oversampling}, Hybrid input-output (HIO) \cite{fienup1987phase} and error reduction (ER) \cite{fienup1982phase} algorithm. 

\subsection{Experimental Results} 

\textbf{Linear Inverse Problems:} We carried out experiments on FFHQ and ImageNet datasets. We show qualitative samples of each task in Figure \ref{fig:all_samples}. DPMC is able to generate valid samples given the noisy, degraded input. Note that, as discussed in \ref{sec:inv_problem}, given the information loss in the image degradation process, exactly restoring the original $\rvx$ is ill-posed. Instead, an effective algorithm should be able to fill in meaningful content that agrees with the noisy observation. This is particularly true for tasks with large information loss, such as inpainting large areas in the image with random or box masks. We show an example in Figure \ref{fig:inpainting_box_seed}, where we present samples generated with different random seeds using the same noisy observation $\rvy$ in the inpainting task with a box-shaped mask. DPMC is capable of generating various samples that agree with the observed part. 

Qualitative results are shown in Table \ref{tab:FFHQ_results} and Table \ref{tab:Imagenet_results}. Following \cite{DPS_ChungKMKY23, dou2023diffusion}, we report the LPIPS score \cite{LPIPS_ZhangIESW18} and FID \cite{FID_HeuselRUNH17} score. The LPIPS score measures the similarity of predicted samples with the ground truth at the single image level. As discussed in \cite{LPIPS_ZhangIESW18}, unlike PSNR and SSIM, which capture shallow, low-level features and might fail to account for many nuances of human perception, LPIPS focuses more on structured information related to human perception. On the other hand, FID measures the distribution differences between generated samples and observations. We believe these two metrics are well-suited to reflect sample performance given the many-to-one mapping nature between $\rvy$ and $\rvx$. As we can see, DPMC achieves similar or better results than DPS on both datasets across all tasks using less NFE. This is also evident in the qualitative comparisons in Figure \ref{fig:illustration_DPS}, where DPS provides blurred samples and our DPMC fills in more vivid details. Compared to other strong baselines, such as FPS or FPS-SMC \cite{dou2023diffusion}, DPMC achieves superior results on most tasks, especially in terms of FID. We have included more qualitative and quantitative comparison with strong baselines in Appendix \ref{sec:morebaselines}. These results demonstrate the effectiveness of introducing MCMC in the sampling process. 

\textbf{Nonlinear Inverse Problem:} We conducted the phase retrieval experiment on the FFHQ dataset. For this, we utilized $T = 200$ intermediate distributions and $K = 6$ MCMC sampling steps at each intermediate distribution. The proposal-exploration step was still applied to the middle 60\% of the sampling steps, corresponding to 920 NFE. Similar to DPS, we observed that the quality of final samples depend on the initialization. Therefore, we followed DPS by generating four different samples and selecting the best one. Our qualitative results are shown in Figure \ref{fig:phase_retrieval}, and the quantitative results are reported in Table \ref{tab:phase_retrieval}. Compared to DPS, our DPMC achieved better LPIPS and a similar FID score.

\begin{table}[ht]
\centering
\caption{Quantitative results of various linear inverse problems on FFHQ $256 \times 256$-1k validation set. \textbf{Bold} denotes the best result for each task and \underline{underline} denotes the second best result.}
\noindent\resizebox{\textwidth}{!}{
\label{tab:FFHQ_results}
\begin{tabular}{cllllllllll}
\toprule
\textbf{Methods} & \multicolumn{2}{c}{\textbf{Super Resolution}} & \multicolumn{2}{c}{\textbf{Inpainting (box)}} & \multicolumn{2}{c}{\textbf{Gaussian Deblur}} & \multicolumn{2}{c}{\textbf{Inpainting (random)}} & \multicolumn{2}{c}{\textbf{Motion Deblur}} \\ \cmidrule(l){2-11} 
& \textbf{FID} & \textbf{LPIPS} & \textbf{FID} & \textbf{LPIPS} & \textbf{FID} & \textbf{LPIPS} & \textbf{FID} & \textbf{LPIPS} & \textbf{FID} & \textbf{LPIPS} \\
\midrule
DPMC (Ours) & \textbf{21.93} & \underline{0.212} & \textbf{19.59} & 0.160 & \textbf{21.34} & \textbf{0.210} & \underline{21.26} & \textbf{0.205} & \textbf{20.73} & \textbf{0.213} \\
\midrule
FPS   \cite{dou2023diffusion}         & 26.66 & \underline{0.212} & \underline{26.13} & \textbf{0.141} & 30.03 & \underline{0.248} & 35.21 & 0.265 & 26.18 & \underline{0.221} \\
FPS-SMC   \cite{dou2023diffusion}     & \underline{26.62} & \textbf{0.210} & 26.51 & \underline{0.150} & \underline{29.97} & 0.253 & 33.10 & 0.275 & \underline{26.12} & 0.227 \\
DPS \cite{DPS_ChungKMKY23}         & 39.35 & 0.214 & 33.12 & 0.168 & 44.05 & 0.257 & \textbf{21.19} & \underline{0.212} & 39.92 & 0.242 \\
DDRM \cite{kawar2022denoising}          & 62.15 & 0.294 & 42.93 & 0.204 & 74.92 & 0.332 & 69.71 & 0.587 & -     & -     \\
MCG \cite{ChungSRY22}           & 87.64 & 0.520 & 40.11 & 0.309 & 101.2 & 0.340 & 29.26 & 0.286 & -     & -     \\
PnP-ADMM  \cite{ChanWE17}     & 66.52 & 0.353 & 151.9 & 0.406 & 90.42 & 0.441 & 123.6 & 0.692 & -     & -     \\
Score-SDE  \cite{Score_base_0011SKKEP21, ChoiKJGY21}    & 96.72 & 0.563 & 60.06 & 0.331 & 109.0 & 0.403 & 76.54 & 0.612 & -     & -     \\
ADMM-TV        & 110.6 & 0.428 & 68.94 & 0.322 & 186.7 & 0.507 & 181.5 & 0.463 & -     & -     \\
\bottomrule
\end{tabular}}
\end{table}

\begin{table}[ht]
\centering
\caption{Quantitative results of various linear inverse problems on ImageNet $256 \times 256$-1k validation set. \textbf{Bold} denotes the best result for each task and \underline{underline} denotes the second best result.}
\noindent\resizebox{\textwidth}{!}{
\label{tab:Imagenet_results}
\begin{tabular}{clllllllllllll}
\toprule
 & \multicolumn{2}{c}{\textbf{Super Resolution}} & \multicolumn{2}{c}{\textbf{Inpainting (box)}} & \multicolumn{2}{c}{\textbf{Gaussian Deblur}} & \multicolumn{2}{c}{\textbf{Inpainting (random)}} & \multicolumn{2}{c}{\textbf{Motion Deblur}} \\ \cmidrule(l){2-11} 
\textbf{Methods} & \textbf{FID} & \textbf{LPIPS} & \textbf{FID} & \textbf{LPIPS} & \textbf{FID} & \textbf{LPIPS} & \textbf{FID} & \textbf{LPIPS} & \textbf{FID} & \textbf{LPIPS} \\ \midrule
DPMC (Ours) & \textbf{31.74} & \textbf{0.307} & \textbf{30.55} & 0.221 & \textbf{33.62} & \textbf{0.318} & \textbf{30.25} & \textbf{0.292} & \textbf{30.88} & \textbf{0.303}\\
\midrule
FPS  \cite{dou2023diffusion}            & 47.32        & 0.329          & \underline{33.19}        & \textbf{0.204}          & 54.41        & \underline{0.396}          & 42.68        & 0.325          & 52.22        & 0.370          \\
FPS-SMC \cite{dou2023diffusion}        & \underline{47.30}        & \underline{0.316}          & 33.24        & \underline{0.212}          & \underline{54.21}        & 0.403          & 42.77        & 0.328          & \underline{52.16}        & \underline{0.365}          \\
DPS \cite{DPS_ChungKMKY23}             & 50.66        & 0.337          & 38.82        & 0.262          & 62.72        & 0.444          & \underline{35.87}        & \underline{0.303}          & 56.08        & 0.389          \\
DDRM  \cite{kawar2022denoising}           & 59.57        & 0.339          & 45.95        & 0.245          & 63.02        & 0.427          & 114.9        & 0.665          & -            & -              \\
MCG  \cite{ChungSRY22}            & 144.5        & 0.637          & 39.74        & 0.330          & 95.04        & 0.550          & 39.19        & 0.414          & -            & -              \\
PnP-ADMM   \cite{ChanWE17}       & 97.27        & 0.433          & 78.24        & 0.367          & 100.6        & 0.519          & 114.7        & 0.677          & -            & -              \\
Score-SDE  \cite{Score_base_0011SKKEP21, ChoiKJGY21}      & 170.7        & 0.701          & 54.07        & 0.354          & 120.3        & 0.667          & 127.1        & 0.659          & -            & -              \\
ADMM-TV          & 130.9        & 0.523          & 87.69        & 0.319          & 155.7        & 0.588          & 189.3        & 0.510          & -            & -              \\ \bottomrule
\end{tabular}}
\end{table}

\begin{figure}[ht]
    \centering 
    \includegraphics[width=.85\textwidth]{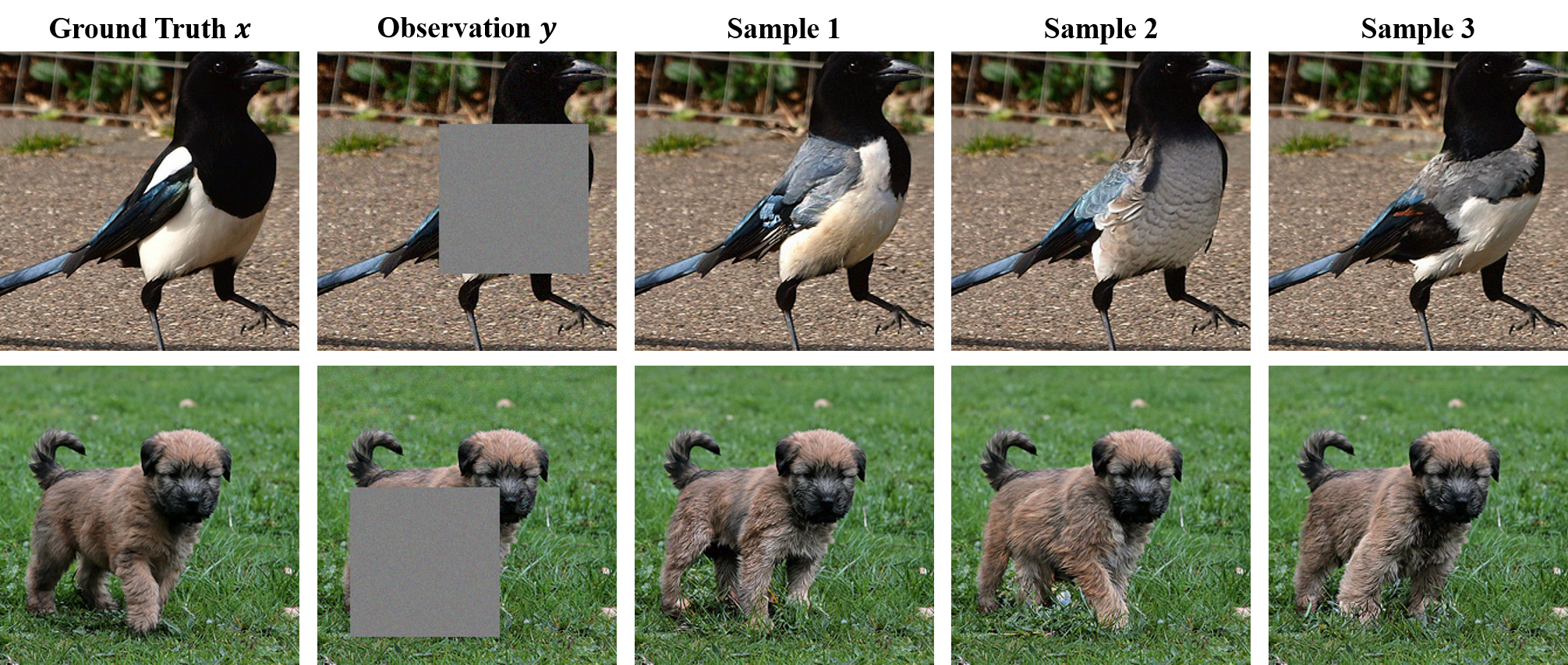}
    \caption{Inpainting results with box-shape mask using different random seed. The first column is the ground truth. The second column is the masked observation. The third to fifth columns are results get by our algorithm under different random seeds.}
    \label{fig:inpainting_box_seed}
\end{figure}

\begin{figure}[ht]
    \centering
    \begin{tabular}{cc}
        \begin{minipage}[c]{0.5\textwidth}
            \centering
            \includegraphics[width=\textwidth]{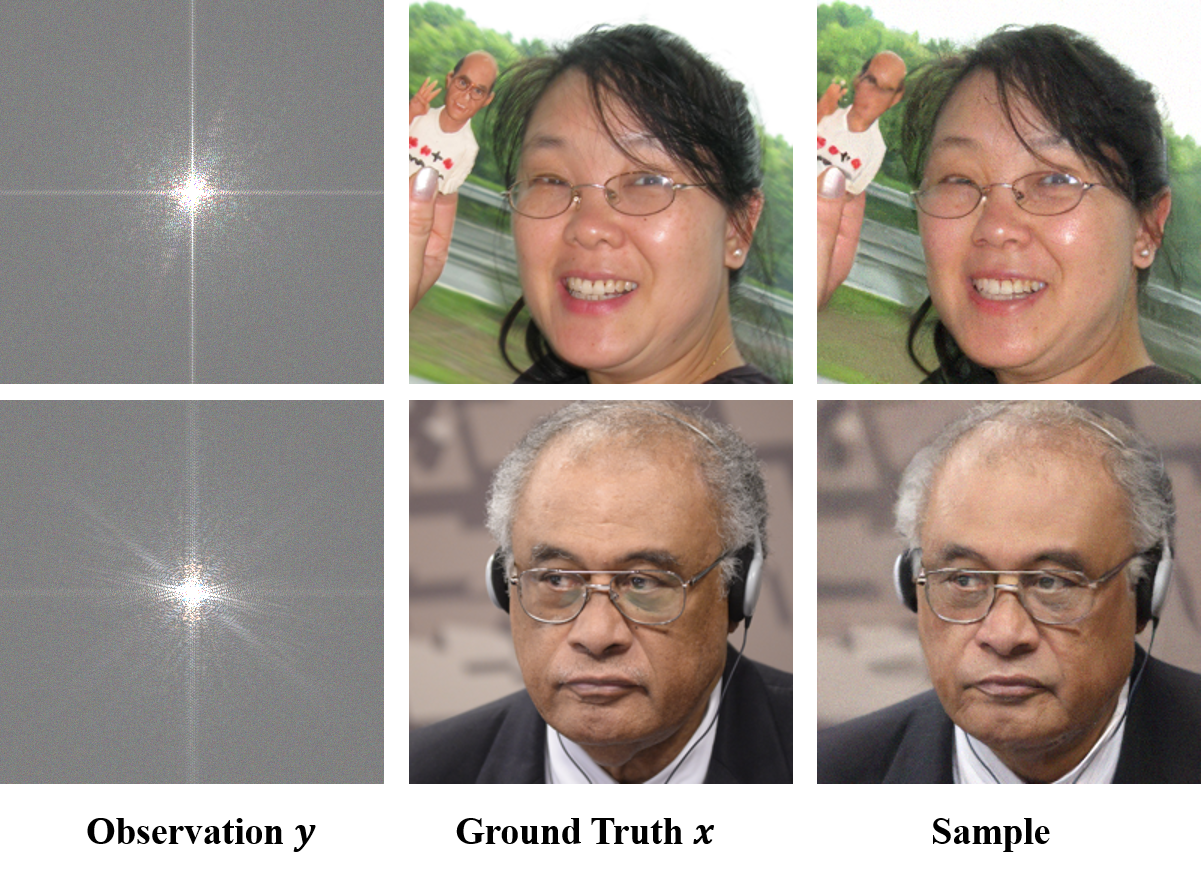}
            \captionof{figure}{Qualitative results of DPMC on FFHQ $256 \times 256$ phase retrieval task.}
            \label{fig:phase_retrieval}
        \end{minipage} &
        \begin{minipage}[c]{0.45\textwidth}
            \centering
            \captionof{table}{Quantitative results of DPMC on FFHQ $256 \times 256$ phase retrieval task.}
            \begin{tabular}{cll}
                \toprule
                \textbf{Methods} & \textbf{FID} & \textbf{LPIPS}\\
                \midrule
                DPMC (Ours) & \underline{55.64} & \textbf{0.372}  \\
                DPS \cite{DPS_ChungKMKY23} & \textbf{55.61} & \underline{0.399} \\
                OSS \cite{rodriguez2013oversampling} & 137.7 & 0.635 \\
                HIO \cite{fienup1987phase} & 96.40 & 0.542 \\
                ER \cite{fienup1982phase} & 214.1 & 0.738 \\
                \bottomrule
            \end{tabular}
            \label{tab:phase_retrieval}
        \end{minipage}
    \end{tabular}
\end{figure}
\vspace{-0.15in}
\subsection{Ablation Study}
We conducted ablation studies on critical parameters of DPMC using the Gaussian deblur task on the FFHQ dataset. Table \ref{tab:ab_K} examines the impact of the number of MCMC sampling steps. Table \ref{tab:ab_T} explores the influence of the number of intermediate distributions. Table \ref{tab:ab_schedule} evaluates different weighting schedules by setting $\xi_t \propto f(\bar{\alpha}_t)$, where $f(\cdot)$ ranges from a constant function to $\bar{\alpha}_t^4$. For each ablation setting, we adjusted the task-related parameters $\eta$ and $\xi$ to optimize the current configuration. The results demonstrate that using more sampling steps or intermediate distributions improves performance, with $T=200$ and $K=4$ being a good choice for balancing performance and sample efficiency. Additionally, setting $\xi_t = \xi \bar{\alpha}_t^3$ yields better results compared to constant scheduling or other alternatives.

\begin{table}[H]
\label{tab:ablation}
\centering
\caption{Ablation studies on number of MCMC sampling step $K$, number of intermediate distribution $T$ and $\epsilon_t$ schedule. In \ref{tab:ab_K}, we keep the same number of intermediate distributions and change the number of MCMC sampling steps; in \ref{tab:ab_T}, we keep the same MCMC sampling steps and change the number of intermediate distributions; in \ref{tab:ab_schedule}, we keep the number of MCMC steps and the number of intermediate distributions and change different weighted schedule $\xi_t$.}

\begin{tabular}{cc}
  \begin{subtable}[t]{0.45\textwidth}
    \centering
    \caption{Keep $T=200$, Change $K$}
    \begin{tabular}{ccccc}
      \toprule
      & \textbf{K=1} & \textbf{K=2} & \textbf{K=4} & \textbf{K=6} \\ 
      \midrule
      \textbf{LPIPS} & 0.227 & 0.214 & 0.210 & 0.209 \\ 
      \textbf{FID} & 26.81 & 22.00 & 21.34 & 21.25 \\ 
      \bottomrule
    \end{tabular}
    \label{tab:ab_K}
  \end{subtable}
  &
  \begin{subtable}[t]{0.45\textwidth}
    \centering
    \caption{Keep $K=4$, Change $T$}
    \begin{tabular}{ccccc}
      \toprule
      & \textbf{T=100} & \textbf{T=200} & \textbf{T=300} & \textbf{T=400} \\ 
      \midrule
      \textbf{LPIPS} & 0.220 & 0.210 & 0.209 & 0.207 \\ 
      \textbf{FID} & 21.89 & 21.34 & 21.29 & 21.35 \\ 
      \bottomrule
    \end{tabular}
    \label{tab:ab_T}
  \end{subtable}
\end{tabular}
\par\bigskip
\vspace{-0.15in}
\begin{subtable}[t]{0.9\textwidth}
  \centering
  \caption{Keep $T=200$, $K=4$, Change $\xi_t$ schedule.}
  \begin{tabular}{cccccc}
    \toprule
    & \textbf{Constant} & $\bar{\boldsymbol{\alpha}}_t$ & \textbf{$\bar{\boldsymbol{\alpha}}_t^2$} & \textbf{$\bar{\boldsymbol{\alpha}}_t^3$} & \textbf{$\bar{\boldsymbol{\alpha}}_t^4$} \\ 
    \midrule
    \textbf{LPIPS} & 0.232 & 0.219 & 0.213 & 0.210 & 0.209\\ 
    \textbf{FID} & 24.72 & 22.53 & 21.62 & 21.34 & 21.43\\ 
    \bottomrule
  \end{tabular}
  \label{tab:ab_schedule}
\end{subtable}
\vspace{-0.15in}
\end{table}
\section{Conclusion and Future Work}
\vspace{-0.05in}
In this study, we propose DPMC, an algorithm based on a specific formula of the posterior distribution and Annealed MCMC sampling to solve inverse problems. We demonstrate that DPMC results in superior sample quality compared to DPS across various inverse problems with fewer functional evaluations. Our study underscores the benefit of incorporating more sampling steps into each intermediate distribution to encourage exploration. Additionally, it is beneficial to trade in the number of intermediate distributions with the number of MCMC exploration steps. One potential limitation of our current work is the necessity to manually tune the weighting schedule and other hyper-parameters, whose optimal values can vary for different tasks. An interesting future direction is to explore scenarios where an explicit estimation $p(\rvx)$ is provided by an EBM or another likelihood-based estimation technique, enabling the use of more advanced samplers such as Hamiltonian Monte Carlo (HMC) with adaptive step sizes. As a powerful algorithm that can process images, DPMC might have the potential to cause negative social consequences, including deepfakes, misinformation, and privacy breaches. We believe that more research and resources are needed to mitigate these risks. 

\clearpage
\bibliographystyle{neurips_2024}
\bibliography{reference}

\appendix
\newpage
\section{Experimental Details}
In this study, for linear inverse problem, we use $T = 200$ intermediate distributions and $K = 4$ MCMC sampling step at each intermediate distributions as our default setting. For the nonlinear inverse problem, phase retrieval, we use T = 200 intermediate distributions and K = 6 MCMC sampling step at each intermediate distributions. Our main task specific parameters are the Guidance weight $\xi$ and Langevin step size coefficient $\eta$. We show the parameter setting for each task in Table \ref{tab:hyperparam}. 
All of our experiments are carried on a single Nvidia A100 GPU.

\label{sec:exp_detail}
\begin{table}[ht]
\centering
\caption{Hyper-parameter settings for various linear inverse problems. $\xi$: Guidance Weight; $\eta$: Langevin Step Size Coefficient}
\label{tab:hyperparam}
\begin{tabular}{cc}
    \begin{subfigure}[b]{0.45\linewidth}
    \centering
    \caption{FFHQ dataset}
    \label{tab:hyperparam_FFHQ}
    \begin{tabular}{lcc}
    \toprule
    Settings & $\xi$ & $\eta$ \\
    \midrule
    Super Resolution & 3.6e3 & 0.2 \\
    Inpainting (box) & 4.2e3 & 0.5 \\
    Gaussian Deblur & 6e3 & 0.3 \\
    Inpainting (random) & 6e3 & 0.5 \\
    Motion Deblur & 6.6e3 & 0.3 \\
    Phase Retrieval & 1.5e3 & 3.0 \\
    \bottomrule
    \end{tabular}
    \end{subfigure}
     &
    \begin{subfigure}[b]{0.45\linewidth}
    \centering
    \caption{ImageNet dataset}
    \label{tab:hyperparam_ImageNet}
    \begin{tabular}{lcc}
    \toprule
    Settings & $\xi$ & $\eta$ \\
    \midrule
    Super Resolution & 3.3e3 & 0.4 \\
    Inpainting (box) & 5.1e3 & 0.4 \\
    Gaussian Deblur & 4.5e3 & 0.3 \\
    Inpainting (random) & 5.7e3 & 0.5 \\
    Motion Deblur & 6e3 & 0.5 \\
    \bottomrule
    \end{tabular}
    \end{subfigure}
\end{tabular}
\end{table}

\section{Theoretical Analysis}
\label{sec:theory}
In this section, we provide some theoretical analysis on the DPMC model proposed by us. First, we start with some existing results on the KL-convergence of the Langevin MCMC algorithm. Denote $p^*$ as the target distribution over $\R^d$, and $s^*(\bx) := \nabla \log p^*(\bx)$ as its score function. The Langevin MCMC algorithm with step size $\eta$ is given by:
\[\bX_0 \sim p_0,~~\bX_{i+1} = \bX_i + \eta s^*(\bX_i) + \sqrt{2\eta} \cdot\xi_i\]
where $\xi_i \sim \mathcal N(0, \bm I_d)$. Before we state the convergence rate, we propose the strong convexity and Lipschitz continuity assumption that the target distribution $p^*$ needs to satisfy. 
\begin{assumption}[Strong Convexity and Lipschitz Continuity of Potential Function]
Let $U(\bX) = -\log p^*(\bX)$ be the potential function. It has $m$-strong convexity and its gradient has $L$-Lipschitz continuous, i.e.:
\[m\bm I_d \preceq \nabla^2 U(\bX) \preceq L\bm I_d ~~~\text{for}~\forall \bX\in\R^d.\]
\end{assumption}
Under this assumption, \cite{cheng2018convergence} propose the following total variation convergence result of the Langevin MCMC algorithm. 
\begin{lemma}[TV-Convergence of Langevin MCMC Algorithm]
\label{lemma:langevin}
After choosing the step size $\eta = \frac{m\varepsilon^2}{32Ld^2}$, and the number of iterations 
\[K = \frac{32L^2d \log(\TV(p_0\| p^*)/\varepsilon)}{m^2\varepsilon^2},\]
the last iterate distribution $p_K := \mathrm{Law}(\bX_K)$ holds that:
\[\TV\left(p_K \| p_0\right) \leq \varepsilon.\]
\end{lemma}
Notice that in the original Theorem 1 of \cite{cheng2018convergence}, the authors set $p_0 = \mathcal N(0, \frac1m \bm I_d)$. However, the proof does not actually rely on this choice. Besides, we use the corollary of TV-distance convergence instead of the original KL-divergence result since TV distance holds triangular inequality, which benefits our analysis. 

Next, we study the distribution estimation error and provide a convergence guarantee for DPMC proposed by us. Our main result requires the following assumptions on the conditional distribution $p(\rvx_t\mid \rvy)$ as well as the unconditional score estimation error. 

\begin{assumption}[Unconditional Score Estimation Error]
\label{assump:uncond}
For all $k=1,2,\ldots, N$, it holds that:
\[\mathbb E_{\bx\sim p_{kh}} \|\hat{s}_{\theta}(\bx, kh) - \nabla\log p_{kh}(\bx)\|^2 \leq \varepsilon_{\mathrm{score}}^2\]
where $\hat{s}_{\theta}(\bx, t)$ is the pretrained score estimator we plug into our algorithm, and $h = T/N$ is the time step. 
\end{assumption}

\begin{assumption}[Conditional Score Approximation]
\label{assump:cond}
For all $k=1,2,\ldots, N$, we have an upper bound for the TV-distance between the true conditional distribution $p_{kh}(\rvx_{kh}\mid \rvy)$ and our DPS-type approximation $\bar{p}_{kh}(\rvx_{kh} \mid \rvy)$ as follows:
\[\TV\left(\bar{p}_{kh}(\bx_{kh}\mid \rvy)~ \|~ p_{kh}(\bx_{kh}\mid \rvy)\right) \leq \varepsilon_{\mathrm{cond}}.\]
\end{assumption}

\begin{assumption}[Lipschitz Continuity, Strongly Convexity and Bounded Moment of Conditional Score]
\label{assump:others}
For all $t$, the conditional score $\nabla \log p_t(\bx_t\mid \rvy)$ is $L$-Lipschitz continuous, and the potential function $\log p_t(\bx_t\mid \rvy)$ is $m$-strongly convex, which enables us to obtain exponential convergence of Langevin MCMC algorithm. We also assume that
\[\mathbb E_{\bx_t\sim p_t(\cdot \mid \rvy)} \|\nabla \log p_t(\rvy\mid \bx_t)\|^2 \leq U_{\mathrm{cond}}^2.\]
For $t=0$, $p_0(\cdot \mid \rvy)$ has bounded second-order moment, i.e.
\[m_2^2 := \mathbb E_{\rvx_0\sim p_0(\cdot \mid \rvy)} \|\rvx_0\|^2 < \infty.\]
It also leads to the conclusion that $p_0(\rvx_0\mid \rvy)$ has a bounded KL-divergence from the standard Gaussian distribution, i.e.
\[\KL\left(p_0(\rvx_0\mid y) \| \mathcal N(0,\bm I_d)\right) \leq \poly(d).\]
\end{assumption}
Now, we state our main theorem as follows:
\begin{theorem}
\label{thm:main}
Under Assumptions \ref{assump:uncond}-\ref{assump:others}, once our time step $h < 1/L \wedge 1$, we can guarantee that our last-iterate conditional distribution $q_0(\cdot\mid \rvy)$ derived from DPMC is $(\varepsilon+\varepsilon_{\mathrm{cond}})$-close from ground truth with respect to TV-distance:
\[\TV(q_0(\rvx_0\mid \rvy), p_0(\rvx_0\mid \rvy)) \leq \varepsilon + \varepsilon_{\mathrm{cond}}\]
after choosing the step size $\eta$ and the number of inner loops (denoted by $K$) of the Langevin MCMC algorithm at all time steps as follows:
\[\eta = \frac{m\varepsilon^2}{32Ld^2}, ~~K=\frac{32L^2d\cdot \log((\sqrt{\poly(d)\cdot \exp(-T)} + \varepsilon_{\mathrm{cond}})/\varepsilon ~\vee~\varepsilon_{\mathrm{inter}} / \varepsilon)}{m^2\varepsilon^2}.\]
Here, 
\[\varepsilon_{\mathrm{inter}}:= \varepsilon + \varepsilon_{\mathrm{cond}} + C\sqrt{h}\cdot(L\sqrt{dh}+Lm_2 h) + C\sqrt{h}(\varepsilon_{\mathrm{score}} + U_{\mathrm{cond}}), \]
and $C$ is a universal constant. 
\end{theorem}
\begin{proof}
Denote $q_t(\cdot)$ as the probability measure derived by the backward process of diffusion model. For the initial step of backward process, we have $q_T = \mathcal N(0,\bm I_d)$, whose distance from $p_T$ shows the convergence of forward process. According to \cite{chen2022sampling}, the variance-preserving framework leads to exponential convergence of forward process, i.e.
\[\KL\left(p_T(\rvx_T\mid y) \| \mathcal N(0,\bm I_d)\right) \leq \KL\left(p_0(\rvx_0\mid y) \| \mathcal N(0,\bm I_d)\right)\cdot \exp(-T) \leq \poly(d)\cdot \exp(-T). \]
By using Pinsker's Inequality and Assumption \ref{assump:cond}, we have:
\[\TV\left(\bar{p}_T(\rvx_T\mid y) \| \mathcal N(0,\bm I_d)\right)\leq \sqrt{\poly(d)\cdot \exp(-T)} + \varepsilon_{\mathrm{cond}}.\]
Starting from $\tilde{\rvx}_T\sim q_T=\mathcal N(0,\bm I_d)$, we apply Langevin MCMC algorithm with regard to the score function $\nabla \log p_T(\rvx_T\mid \rvy)$. By using Lemma \ref{lemma:langevin}, we can make $\TV(q_T(\rvx_T\mid \rvy), \bar{p}_T(\rvx_T\mid \rvy)) \leq \varepsilon$ by using step size $\eta = \frac{m\varepsilon^2}{32Ld^2}$, and the number of iterations 
\[K = \frac{32L^2d \log(\TV(q_T\| \bar{p}_T(\rvx_T\mid \rvy))/\varepsilon)}{m^2\varepsilon^2}.\]
Next, we apply unconditional backward step as well as Langevin MCMC inner loops to make sure $\TV(q_{kh}(\rvx_{kh}\mid \rvy), \bar{p}_{kh}(\rvx_{kh}\mid \rvy)) \leq \varepsilon$ holds for $\forall k\in [N]$, including the last iterate $k=0$. We prove it by the method of induction. Assume it holds that $\TV\left(q_{(k+1)h}(\rvx_{(k+1)h}\mid \rvy)~\|~ \bar{p}_{(k+1)h}(\rvx_{(k+1)h}\mid \rvy)\right) \leq \varepsilon$, then we immediately have
\[\TV\left(q_{(k+1)h}(\rvx_{(k+1)h}\mid \rvy)~\|~ p_{(k+1)h}(\rvx_{(k+1)h}\mid \rvy)\right) \leq \varepsilon + \varepsilon_{\mathrm{cond}}\]
by using Assumption \ref{assump:cond} as well as the triangular inequality of TV distance. After applying a backward diffusion step with unconditional score estimator $\hat{s}_{\theta}(\cdot, t)$ plugged in and obtain $\bar{\rvx}_{kh}$. During this backward step, the score estimation error is actually
\begin{align*}
&\bE_{\rvx_{kh}\sim p_{kh}(\rvx_{kh}\mid \rvy)} \|\hat{s}_{\theta}(\rvx_{kh})-\nabla \log p_{kh}(\rvx_{kh} \mid \rvy)\|^2 \\
&~~\leq \bE_{\rvx_{kh}}\|\hat{s}_{\theta}(\rvx_{kh})-\nabla \log p_{kh}(\rvx_{kh}) - \nabla \log p_{kh}(\rvy\mid \rvx_{kh})\|^2 \\
&~~\leq 2 \bE_{\rvx_{kh}} \|\hat{s}_{\theta}(\rvx_{kh})-\nabla \log p_{kh}(\rvx_{kh})\|^2 + 2\bE_{\rvx_{kh}} \|\nabla \log p_{kh}(\rvy\mid \rvx_{kh})\|^2\\
&~~\leq 2(\varepsilon_{\mathrm{score}}^2 + U_{\mathrm{cond}}^2) \leq 2(\varepsilon_{\mathrm{score}} + U_{\mathrm{cond}})^2.
\end{align*}
Therefore, we substitute the score estimation error $\varepsilon_{\mathrm{score}}$ in Theorem 2 of \cite{chen2022sampling} with $\varepsilon_{\mathrm{score}} + U_{\mathrm{cond}}$. 
We use the results in \cite{chen2022sampling} as well as Girsanov Theorem, and conclude that:
\begin{align*}
&\TV\left(\mathrm{Law}(\bar{\rvx}_{kh})~\|~p_{kh}(\rvx_{kh}\mid \rvy)\right) - \TV\left(q_{(k+1)h}(\rvx_{(k+1)h}\mid \rvy)~\|~ \bar{p}_{(k+1)h}(\rvx_{(k+1)h}\mid \rvy)\right)\\
&~~\leq \underbrace{C\sqrt{h}\cdot(L\sqrt{dh}+Lm_2 h)}_{\text{discretization error}} + \underbrace{C\sqrt{h}(\varepsilon_{\mathrm{score}} + U_{\mathrm{cond}})}_{\text{score estimation error}}
\end{align*}
where $C$ is a universal constant, which leads to
\[\TV\left(\mathrm{Law}(\bar{\rvx}_{kh})~\|~p_{kh}(\rvx_{kh}\mid \rvy)\right) \leq \varepsilon + \varepsilon_{\mathrm{cond}} + C\sqrt{h}\cdot(L\sqrt{dh}+Lm_2 h) + C\sqrt{h}(\varepsilon_{\mathrm{score}} + U_{\mathrm{cond}}) := \varepsilon_{\mathrm{inter}}. \]
As the initial step of Langevin MCMC inner loops at time $t=kh$, we apply the exponential convergence (Lemma \ref{lemma:langevin}) and show that we can make $\TV(q_{kh}(\rvx_{kh}\mid \rvy), \bar{p}_{kh}(\rvx_{kh}\mid \rvy)) \leq \varepsilon$ and complete the induction by letting the step size $\eta = \frac{m\varepsilon^2}{32Ld^2}$ and the number of iterations
\[K = \frac{32L^2d\cdot \log(\varepsilon_{\mathrm{inter}} / \varepsilon)}{m^2\varepsilon^2}.\]
To sum up, we can guarantee that $\TV(q_0(\rvx_0\mid \rvy), \bar{p}_0(\rvx_0\mid \rvy)) \leq \varepsilon$ by choosing the step size $\eta$ and the number of inner loop $K$ of Langevin MCMC algorithm as follows:
\[\eta = \frac{m\varepsilon^2}{32Ld^2}, ~~K=\frac{32L^2d\cdot \log((\sqrt{\poly(d)\cdot \exp(-T)} + \varepsilon_{\mathrm{cond}})/\varepsilon ~\vee~\varepsilon_{\mathrm{inter}} / \varepsilon)}{m^2\varepsilon^2}\]
where 
\[\varepsilon_{\mathrm{inter}}:= \varepsilon + \varepsilon_{\mathrm{cond}} + C\sqrt{h}\cdot(L\sqrt{dh}+Lm_2 h) + C\sqrt{h}(\varepsilon_{\mathrm{score}} + U_{\mathrm{cond}}). \]
It finally comes to our conclusion as $\TV(\bar{p}_0(\rvx_0\mid \rvy)~\|~p_0(\rvx_0\mid \rvy)) \leq \varepsilon+{\mathrm{cond}}$. 
\end{proof}

\section{More Experimental Results}
\label{sec:more_exp}
\subsection{Sampling time}
In Table \ref{tab:running time}, we report the time for generating one sample with DPMC default setting, which sets $T=200$, $K=4$ and applies MCMC sampling steps in the middle 60\% intermediate distributions. (See section \ref{sec:exp_setting} for more details.) The time is tested on the FFHQ dataset using a single Nvidia A100 GPU.

\begin{table}[h]
    \centering
    \caption{Running Times of Different Methods for Generating one Sample on FFHQ-1k Validation Dataset}
    \label{tab:running time}
    \begin{tabular}{cc}
        \toprule
        \textbf{Model} & \textbf{Running Time (Seconds)} \\
        \midrule
        DPMC(ours) & 54.63\\
        FPS \cite{dou2023diffusion} & 33.07 \\
        FPS-SMC (with $M=5$) \cite{dou2023diffusion} & 57.88\\
        DPS \cite{DPS_ChungKMKY23} & 70.42 \\
        Score-SDE \cite{Song20220011S0E22} & 32.93 \\
        DDRM \cite{kawar2022denoising} & 2.034 \\
        MCG \cite {ChungSRY22} & 73.16 \\
        PnP-ADMM \cite{ChanWE17} & 3.595 \\
        $\Pi$GDM \cite{SongVMK23} & 33.18 \\
        \bottomrule
    \end{tabular}
\end{table}

\subsection{More Samples for Each Inverse Problems}
In Figure \ref{fig:sp_more}, \ref{fig:gb}, \ref{fig:mb}, \ref{fig:inpainting_box}, \ref{fig:inpainting_random}, we show more samples for each inverse problem. DPMC works well across different tasks and datasets.

\begin{figure}[h]
    \centering
    \includegraphics[width=0.95\textwidth]{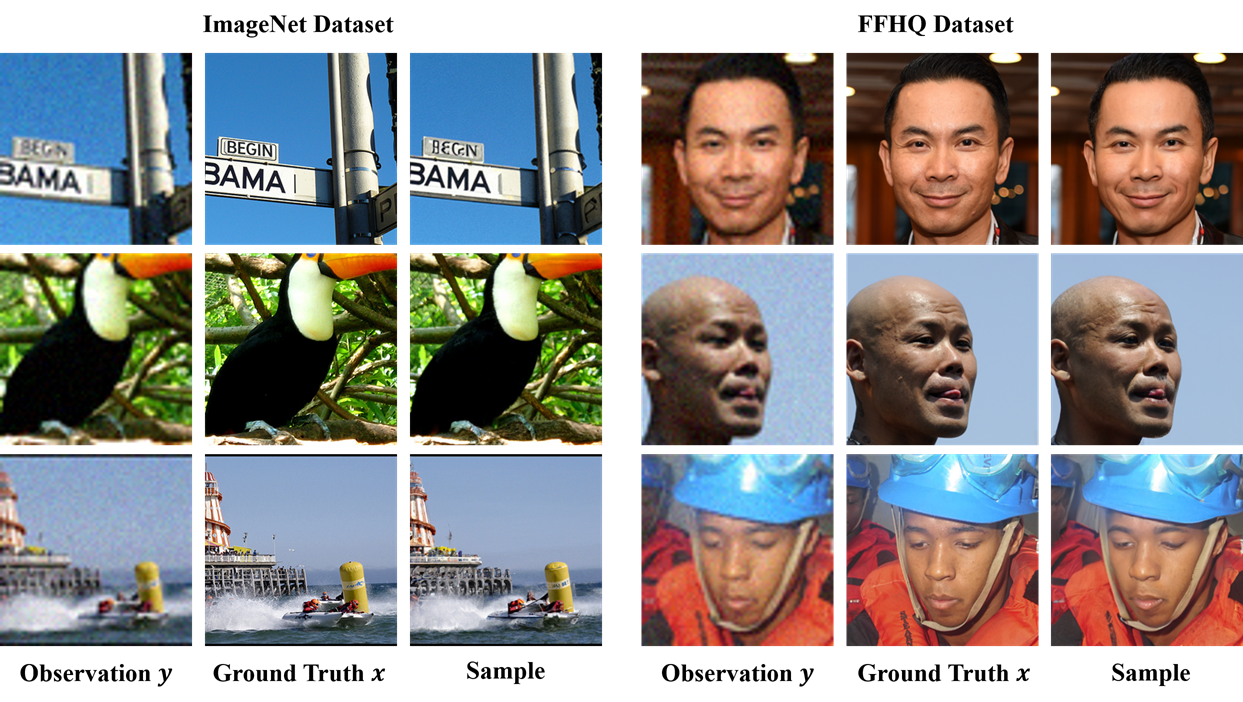}
    \caption{More Results for Super Resolution Task}
    \label{fig:sp_more}
\end{figure}

\begin{figure}[h]
    \centering
    \includegraphics[width=0.95\textwidth]{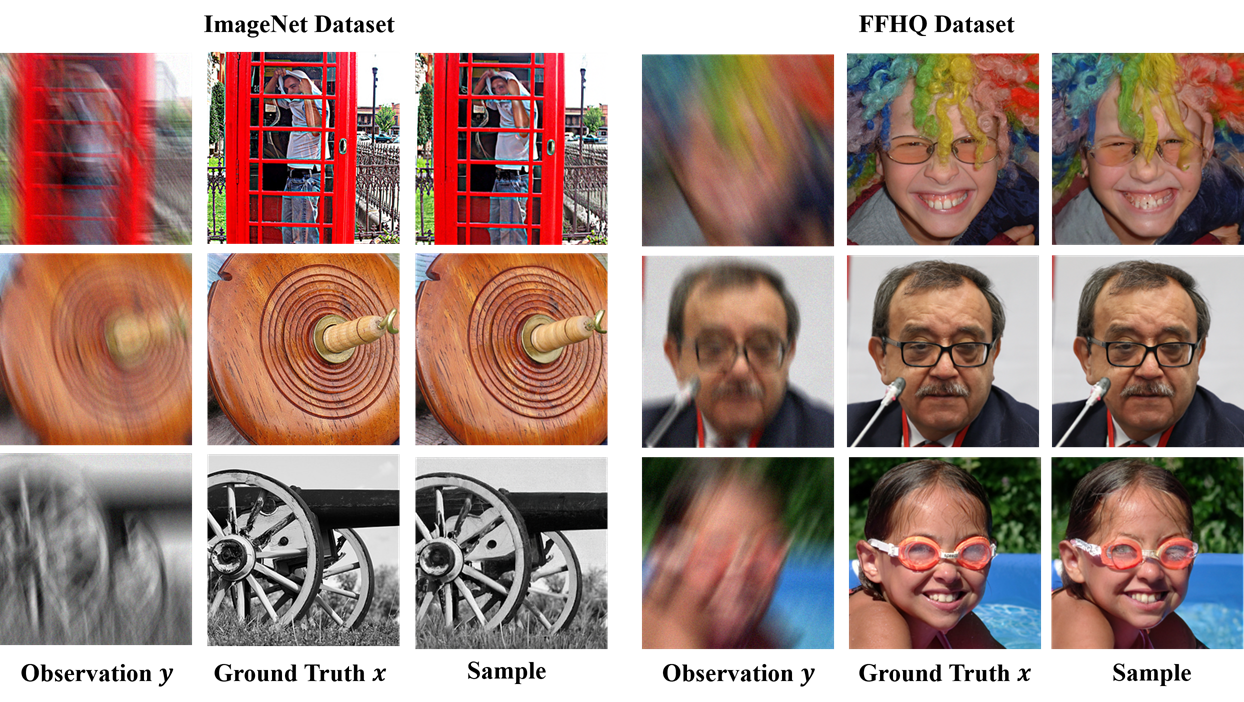}
    \caption{More Results for Motion Deblur Task}
    \label{fig:mb}
\end{figure}

\begin{figure}[h]
    \centering
    \includegraphics[width=0.95\textwidth]{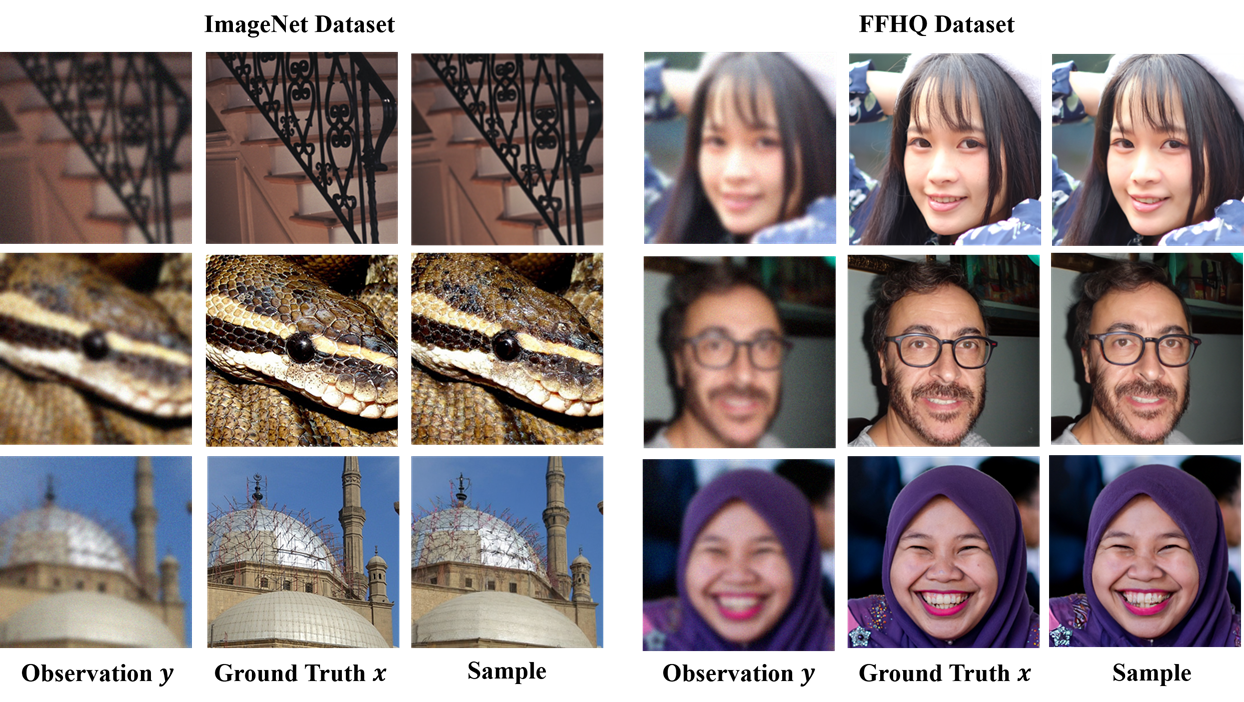}
    \caption{More Results for Gaussian Deblur Task}
    \label{fig:gb}
\end{figure}

\begin{figure}[h]
    \centering
    \includegraphics[width=0.95\textwidth]{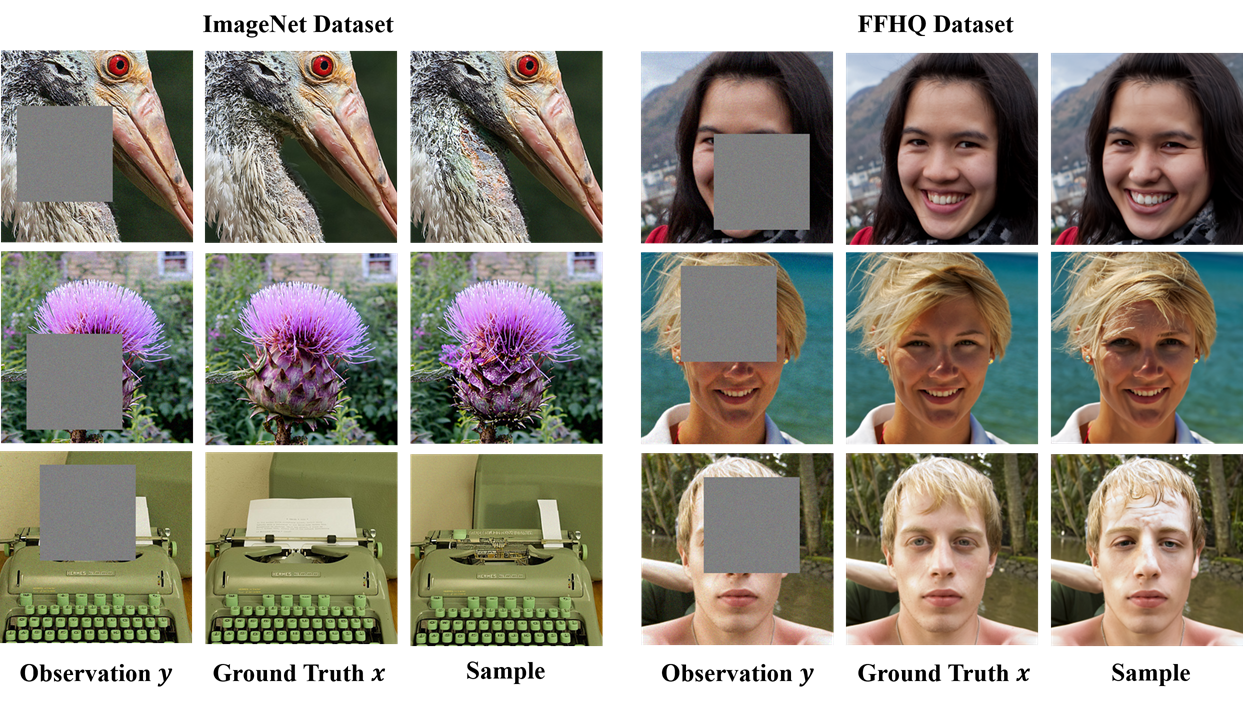}
    \caption{More Results for Inpainting (Box) Task}
    \label{fig:inpainting_box}
\end{figure}

\begin{figure}[h]
    \centering
    \includegraphics[width=0.95\textwidth]{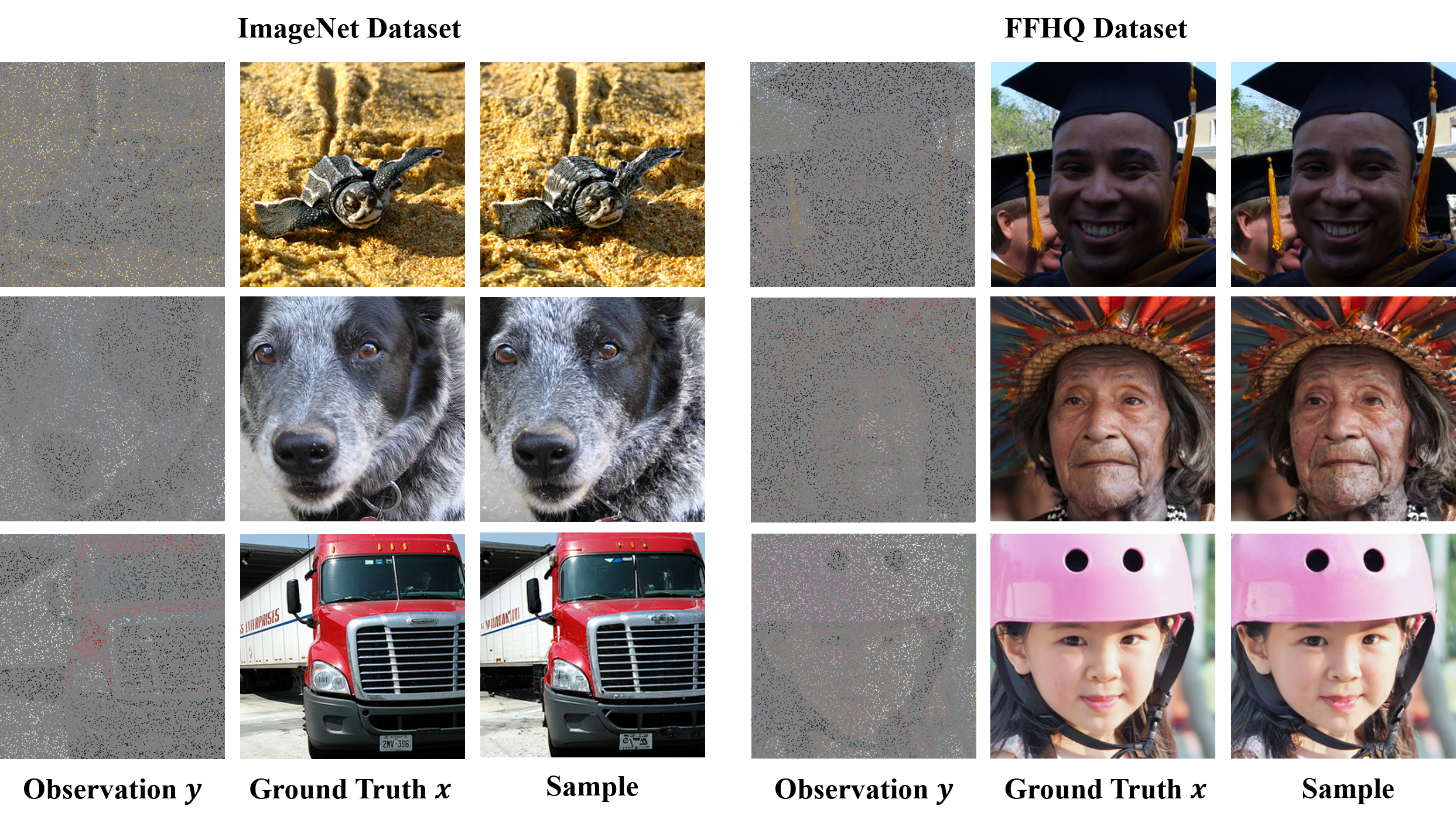}
    \caption{More Results for Inpainting (Random) Task}
    \label{fig:inpainting_random}
\end{figure}

\subsection{More baselines}
\label{sec:morebaselines}
In this section we compare DPMC with more strong baselines. More precisely, we compare our model with $\Pi$GDM \cite{SongVMK23}, DDNM\cite{DDNM_WangYZ23}, RED-Diff \cite{mardani2023variational}, DiffPir \cite{DiffPir_ZhuZLCWTG23} and CCDF \cite{CCDF_ChungSY22}. We conducted experiments for RED-diff, $\Pi$GDM and DDNM. And we obtained the reported results of DiffPir and CCDF from their original papers. The quantitative results are shown in Table \ref{tab:FFHQ_results_all} and Table \ref{tab:Imagenet_results_all}.

We have observed that the results of RED-diff and $\Pi$GDM are significantly influenced by whether noise is added to the degradation process. We observed that if noise is added, the optimized samples might exhibit noise-like artifacts and sometimes completely fail. Therefore, we include both the noisy version (similar to our setting) and the noiseless version (denoted as $\sigma = 0$ ) here. Note that due to setting differences, the noiseless results are listed here for reference only.

From the quantitative results, we can see that our DPMC demonstrates comparable or better performance compared to strong baselines in the noisy measurement case ($\sigma = 0.05$), particularly in terms of FID scores. Our results remain robust even when compared with the noiseless results of RED-diff and $\Pi$GDM. 

We also include qualitative comparision in Figure \ref{fig:quant_res_all}. The main advantage of our DPMC is that our samples contain more vivid details and can accept noisy settings ($\sigma > 0$). The samples from RED-diff and $\Pi$GDM might contain noise-like artifacts in noisy settings. The results of FPS-SMC, DDNM, $\Pi$GDM ($\sigma=0.0$) and RED-diff ($\sigma=0.0$) don't contain noise-like textures; however, these samples tend to be blurrier than ours. 

\begin{table}[ht]
\centering
\caption{Quantitative results of various linear inverse problems on FFHQ $256 \times 256$-1k validation set with more strong baselines.}
\noindent\resizebox{\textwidth}{!}{
\label{tab:FFHQ_results_all}
\begin{tabular}{cllllllllll}
\toprule
\textbf{Methods} & \multicolumn{2}{c}{\textbf{Super Resolution}} & \multicolumn{2}{c}{\textbf{Inpainting (box)}} & \multicolumn{2}{c}{\textbf{Gaussian Deblur}} & \multicolumn{2}{c}{\textbf{Inpainting (random)}} & \multicolumn{2}{c}{\textbf{Motion Deblur}} \\ \cmidrule(l){2-11} 
& \textbf{FID} & \textbf{LPIPS} & \textbf{FID} & \textbf{LPIPS} & \textbf{FID} & \textbf{LPIPS} & \textbf{FID} & \textbf{LPIPS} & \textbf{FID} & \textbf{LPIPS} \\
\midrule
DPMC (Ours) & 21.93 & 0.212 & 19.59 & 0.160 & 21.34 & 0.210 & 21.26 & 0.205 & 20.73 & 0.213 \\
\midrule
DDNM\cite{DDNM_WangYZ23} & 26.64 & 0.214 & 25.97 & 0.150 & 28.69 & 0.212 & 28.71 & 0.201 & - & - \\
RED-Diff \cite{mardani2023variational} & 89.13 & 0.435 & - & - & 37.35 & 0.255 & - & - & - & - \\
$\Pi$GDM \cite{SongVMK23} & 29.59 & 0.214 & - & - & 431.83 & 0.887 & - & - & - & - \\
CCDF \cite{CCDF_ChungSY22} & 60.90 & - & 49.77 & - & - & - & - & - & -\\
DiffPir \cite{DiffPir_ZhuZLCWTG23} & - & 0.260 & - & 0.236 & - & - & - & - & - & 0.255 \\
\midrule
RED-Diff $(\sigma = 0.0)$ \cite{mardani2023variational} & 39.68 & 0.185 & - & - & 30.54 & 0.161 & - & - & - & - \\
$\Pi$GDM  $(\sigma = 0.0)$ \cite{SongVMK23} & 39.61 & 0.207 & - & - & 34.52 & 0.140 & - & - & - & - \\
\bottomrule
\end{tabular}}
\end{table}

\begin{table}[ht]
\centering
\caption{Quantitative results of various linear inverse problems on ImageNet $256 \times 256$-1k validation set with more strong baselines.}
\noindent\resizebox{\textwidth}{!}{
\label{tab:Imagenet_results_all}
\begin{tabular}{clllllllllllll}
\toprule
 & \multicolumn{2}{c}{\textbf{Super Resolution}} & \multicolumn{2}{c}{\textbf{Inpainting (box)}} & \multicolumn{2}{c}{\textbf{Gaussian Deblur}} & \multicolumn{2}{c}{\textbf{Inpainting (random)}} & \multicolumn{2}{c}{\textbf{Motion Deblur}} \\ \cmidrule(l){2-11} 
\textbf{Methods} & \textbf{FID} & \textbf{LPIPS} & \textbf{FID} & \textbf{LPIPS} & \textbf{FID} & \textbf{LPIPS} & \textbf{FID} & \textbf{LPIPS} & \textbf{FID} & \textbf{LPIPS} \\ \midrule
DPMC (Ours) & 31.74 & 0.307 & 30.55 & 0.221 & 33.62 & 0.318 & 30.25 & 0.292 & 30.88 & 0.303\\
\midrule
RED-Diff \cite{mardani2023variational} & 82.62 & 0.471 & - & - & 39.11 & 0.319 & - & - & - & - \\
$\Pi$GDM \cite{SongVMK23} & 43.55 & 0.343 & - & - & 371.33 & 0.813 & - & - & - & - \\
DiffPir \cite{DiffPir_ZhuZLCWTG23} & - & 0.371 & - & 0.355 & - & - & - & - & - & 0.366 \\
\midrule
RED-Diff $(\sigma = 0.0)$ \cite{mardani2023variational} & 45.17 & 0.304 & - & - & 32.29 & 0.232 & - & - & - & - \\
$\Pi$GDM  $(\sigma = 0.0)$ \cite{SongVMK23} & 50.21 & 0.342 & - & - & 32.99 & 0.200 & - & - & - & - \\
\bottomrule
\end{tabular}}
\end{table}

\begin{figure}[ht]
    \begin{subfigure}[b]{0.98\textwidth}
        \centering
        \includegraphics[width=\textwidth]{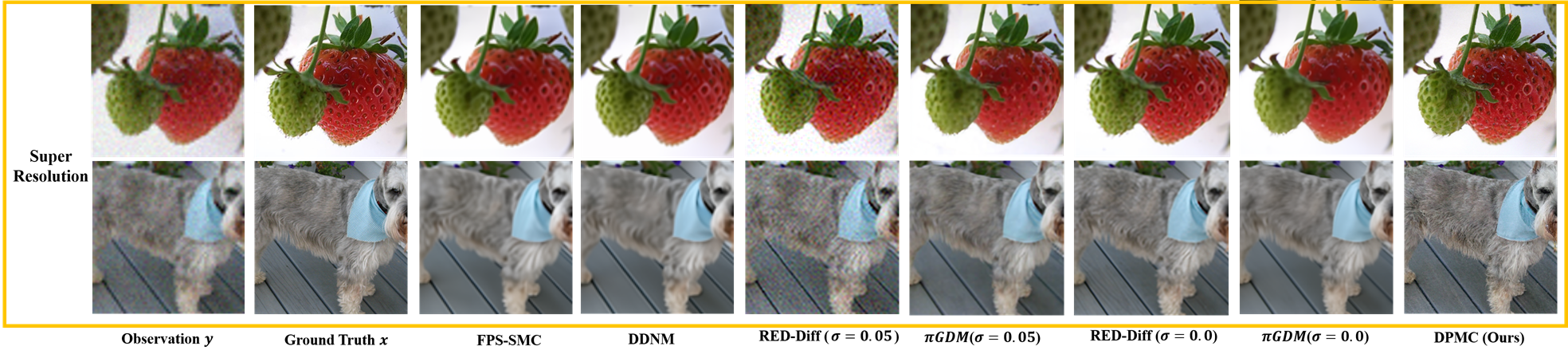}
        \label{fig:subfig1}
    \end{subfigure}
    
    \begin{subfigure}[b]{0.98\textwidth}
        \centering
        \includegraphics[width=\textwidth]{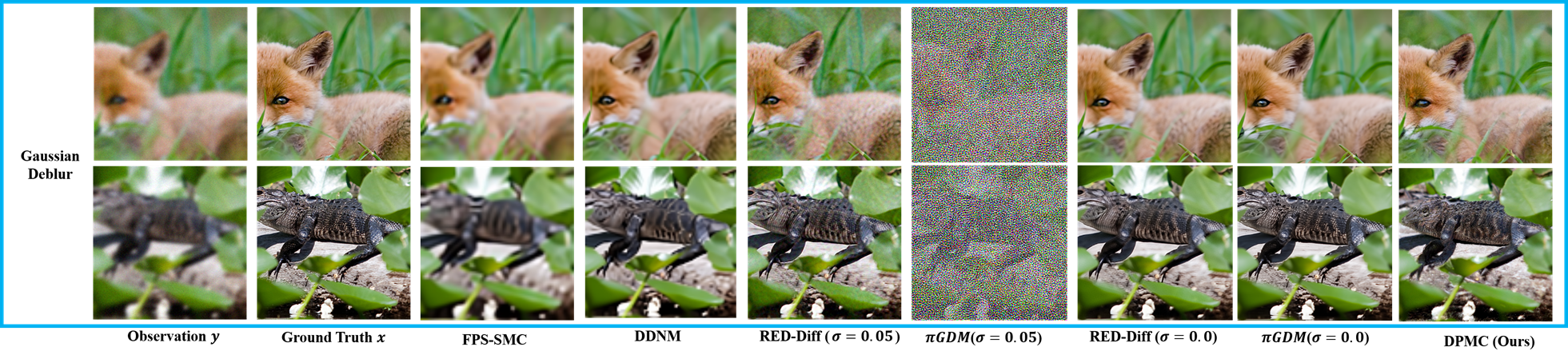}
        \label{fig:subfig3}
    \end{subfigure}
    \caption{Qualitative comparison of different methods on ImageNet $256 \times 256$-1k dataset.}
    \label{fig:quant_res_all}
\end{figure}

\end{document}